%% file: main.tex
\documentclass{article}

\usepackage{microtype}
\usepackage{graphicx}
\usepackage{subfigure}
\usepackage{booktabs} 
\usepackage{enumitem}

\usepackage{hyperref}

\usepackage{algorithmic}

\usepackage[accepted]{icml2024}

\usepackage{amsmath}
\usepackage{amssymb}
\usepackage{mathtools}
\usepackage{amsthm}

\usepackage{multirow}
\usepackage{tabu, booktabs}
\usepackage{makecell}

\usepackage{bm}
\usepackage{arydshln}

\usepackage[capitalize,noabbrev]{cleveref}

\theoremstyle{plain}
\newtheorem{theorem}{Theorem}[section]
\newtheorem{proposition}[theorem]{Proposition}
\newtheorem{lemma}[theorem]{Lemma}

\theoremstyle{definition}

\theoremstyle{remark}
\newtheorem{remark}[theorem]{Remark}

\newtheorem*{theorem*}{Theorem}
\newtheorem*{proposition*}{Proposition}

\usepackage[textsize=tiny]{todonotes}

\input{commands}

\renewcommand{\b}[1]{{\textbf{#1}}}
\newcommand{\bb}[1]{{\underline{#1}}}

\icmltitlerunning{
LIDAO: Towards Limited Interventions for Debiasing (Large) Language Models
}

\begin{document}

\twocolumn[
\icmltitle{
LIDAO: Towards Limited Interventions for Debiasing (Large) Language Models
}




\begin{icmlauthorlist}
\icmlauthor{Tianci Liu}{1}
\icmlauthor{Haoyu Wang}{1}
\icmlauthor{Shiyang Wang}{1}
\icmlauthor{Yu Cheng}{2}
\icmlauthor{Jing Gao}{1}
\end{icmlauthorlist}

\icmlaffiliation{1}{Purdue University}
\icmlaffiliation{2}{The Chinese University of Hong Kong}

\icmlcorrespondingauthor{Yu Cheng}{chengyu@cse.cuhk.edu.hku}
\icmlcorrespondingauthor{Jing Gao}{jinggao@purdue.edu}

\icmlkeywords{Fairness, Language Models}

\vskip 0.3in
]



\printAffiliationsAndNotice{}  

\begin{abstract}
\input{subfiles/0_abstract}
\end{abstract}

\input{subfiles/1_intro}
\input{subfiles/3_method}

\input{subfiles/4_experiment}

\input{subfiles/2_background}

\input{subfiles/5_conclusion}

\section*{Acknowledgement}
This work is supported in part by the US National Science Foundation under grant NSF IIS-2141037 and NSF IIS-2226108. Any opinions, findings, and conclusions or recommendations expressed in this material are those of the author(s) and do not necessarily reflect the views of the National Science Foundation.

\section*{Impact Statement}

This paper presents work whose goal is to advance the field of machine learning.
The datasets used in this work and sample generations may include potentially offensive or biased contents. 
We acknowledge that exposure to these contents can be unpleasant or uncomfortable. 
However, the use of these materials is to better understand, examine, and mitigate the harmful generations from language models.

\bibliography{ref}
\bibliographystyle{icml2024}

\newpage
\appendix
\onecolumn

\input{subfiles/6_appendix}

\end{document}

%% file: commands.tex
\newcommand{\beq}{\vspace{0mm}\begin{equation}}
\newcommand{\eeq}{\vspace{0mm}\end{equation}}
\newcommand{\beqs}{\vspace{0mm}\begin{eqnarray}}
\newcommand{\eeqs}{\vspace{0mm}\end{eqnarray}}
\newcommand{\barr}{\begin{array}}
\newcommand{\earr}{\end{array}}

\newcommand{\cv}[0]{{\boldsymbol{c}}}

\newcommand{\xv}{\boldsymbol{x}}

\usepackage{booktabs}
\usepackage{array}
\newcolumntype{L}[1]{>{\raggedright\let\newline\\\arraybackslash\hspace{0pt}}m{#1}}
\newcolumntype{C}[1]{>{\centering\let\newline\\\arraybackslash\hspace{0pt}}m{#1}}
\newcolumntype{R}[1]{>{\raggedleft\let\newline\\\arraybackslash\hspace{0pt}}m{#1}}

\newcommand{\E}{\mathbb{E}}

\DeclareUnicodeCharacter{2212}{\ensuremath{-}}

%% file: subfiles/0_abstract.tex
\textit{\textbf{Warning}: this paper contains model outputs exhibiting offensiveness and biases.}

Large language models (LLMs) have achieved impressive performance on various natural language generation tasks.
Nonetheless, they suffer from generating negative and harmful contents that are biased against certain demographic groups (e.g., female), raising severe fairness concerns. 
As remedies, 
prior works intervened the generation by removing attitude or demographic information, 
inevitably degrading the generation quality and resulting in notable \textit{fairness-fluency} trade-offs. 
However, it is still under-explored to what extent the fluency \textit{has to} be affected in order to achieve a desired level of fairness. 
In this work, we conduct the first formal study from an information-theoretic perspective.  
We show that previous approaches are excessive for debiasing and propose LIDAO, a general framework to debias a (L)LM at a better fluency provably. 
We further robustify LIDAO in adversarial scenarios, where a carefully-crafted prompt may stimulate LLMs exhibiting instruction-following abilities to generate texts with fairness issue appears only when the prompt is also taken into account. 
Experiments on three LMs ranging from 0.7B to 7B parameters demonstrate the superiority of our method.

%% file: subfiles/1_intro.tex
\section{Introduction}
\label{sec:intro}


Language models (LMs) parameterized by deep neural networks \citep{vaswani2017attention,lewis2019bart,radford2019language,brown2020language} have thrived in producing fluent and meaningful texts on a variety of natural language generation tasks \citep{see2019massively,ji2023survey}.
Recently, researchers further applied LMs to more diverse classification tasks by transforming them to generative frames \citep{raffel2020exploring}. These successes underscore the versatility of LMs, establishing them as the foundations for different natural language processing applications \citep{bommasani2021opportunities,zhou2023comprehensive}. 
In addition, with model sizes continually increasing, 
large language models (LLMs) have demonstrated unprecedented abilities to follow natural language instructions \citep{dong2022survey,ouyang2022training}. 
These abilities empower the zero-shot adaptation of LLMs to unseen tasks \citep{kojima2022large}, paving the way towards artificial general intelligence \citep{bubeck2023sparks}.

Notwithstanding, 
despite their remarkable performance, LMs suffer from the \textit{fairness} issue, i.e., 
they may generate negative texts that are biased against under-represented demographic groups (e.g., {female}) in our society \citep{sheng2019woman}. 
For instance, GPT-2 \citep{radford2019language} tends to generate more negative texts towards females \citep{huang2019reducing}. 
Such social biases, termed as \textit{global biases} \citep{sheng2020towards}, stem from the real world corpora due to historical reasons \citep{basta2019evaluating}. 
Unsurprisingly, LMs reproduce or amplify the biases from the data whereon they are trained \citep{gehman2020realtoxicityprompts,schick2021self}.



As remedies, a vast amount of bias mitigation approaches have been proposed.
Early attempts of fine-tuning with clean data proved effective on small LMs \citep{lu2020gender,saunders2020reducing,bender2021dangers}. However, they fall short on LLMs due to the prohibitive resource requirement for clean data curation and documentation \citep{schick2021self}. 
Recently, the lightweight {decoding-time intervention} paradigm, which directly modifies the text generation process in a LM, showed promising results and attracted great attention \citep{liu2021dexperts, yang2023unified}. 
Conceptually, these interventions imposed certain constraints to the generation process. 
For instance, in order to debias sentiment on different demographic groups, \citet{sheng2020towards} steered the LM to generate positive texts in all cases. 
Similarly, \citet{liang2021towards} sought to remove gender information during the text generation to avoid gender bias. 
Constraints as such, which we dub \textit{constant constraints}, hinder the LMs from generating sentiment- or gender-specific texts, inevitably resulting in degradation on the generation quality \citep{liu2021dexperts,fatemi2023improving}.

\textit{Fairness-fluency trade-off} refers to this phenomenon of increasing fairness at the cost of fluency, and has been widely observed in the literature \citep{liang2021towards}.
Noting this potential drawback, fluency plays a pivotal role in evaluating the practicability of different debiasing methods \citep{yang2023unified, wang2023toward}. 
However, no formal studies has been conducted on to what extent the fluency \textit{has to} be affected in order to achieve a desired level of fairness,
and it is still an open question \textit{whether existing methods can achieve the Pareto optimality, in the sense that no fluency improvement can be made without sacrificing the fairness. }


In this paper, we take the first step and provide a negative answer through an information-theoretic analysis on mitigating the (global) bias in LMs. 
In particular, we quantify the bias in text generation with mutual information (MI) between some \textit{global property} and \textit{demographic group}. 
Here a global property can be any characteristic that may raise fairness concern such as \textit{sentiment} \citep{huang2019reducing} or \textit{regard} \citep{sheng2019woman}. 
Built upon this formulation, we theoretically show that the constant constraints in existing approaches are in fact excessive and can be weakened. 
Particularly, we formalize this intuition: \textit{Instead of intervening the generation of every word in a text, only words that may raise fairness issues need care.}
Steered by the theoretical analysis, we propose a principled framework named LIDAO for bias mitigation with limited interventions.
To be specific, in LIDAO, a word is allowed to be relevant with the (global) property or the demographic group, so long as it is independent with the other. Through this relaxation, LMs are allowed to generate property- or demographic group-specific content, thereby being able to produce more coherent and fluent texts.  
\textit{To our best knowledge, we provide the first formal study on debiasing LMs targetting at better fairness-fluency trade-offs. 
}
Sec \ref{sec:method} details these results. 

The proposed LIDAO provably removes any bias in the generated text. 
However, for LLMs that exhibit strong instruction-following (IF) abilities \citep{ouyang2022training}, 
a carefully-crafted adversarial prompt may instruct or stimulate them to generate a short text where the fairness issue appears \textit{only if} the context from the prompt is taken into account. As such, LIDAO may fall short to apply due to the absence of, for instance, the mention of any demographic group in the generation. 
To enhance the applicability of LIDAO in this adversarial scenario, 
we extend the bias formulation and bound it by two parts: a generation bias solvable with LIDAO, 
{and the extent to which one can infer the demographic group from the generation alone without seeing the prompts.}
We propose a heuristic to minimize the second component by utilizing a LM's IF ability as well. 
This design makes our solution highly lightweight and effective. 
Sec \ref{sec:adversarial} discusses this extension in details. 





Experimental results in Sec \ref{sec:experiment} demonstrate the efficacy of our method in achieving better fairness-fluency trade-offs when debiasing small and large LMs on three representative tasks. 
In the remaining part of this paper, 
we review related works in Sec \ref{sec:background}, and conclude the paper in Sec \ref{sec:conclusion}.

%% file: subfiles/3_method.tex
\section{Proposed Method}
\label{sec:method}

Grounded in a theoretic analysis, 
we propose LIDAO to debias language models (LMs) with limited interventions, thereby achieving a better fairness-fluency tradeoff. 
LM preliminaries are also provided.

\subsection{Preliminaries}

Given a sentence 
$\xv = (x_1, \dots, x_T)$ of length $T$ and each $x_t \in \mathcal V$ is a token from the vocabulary $\mathcal V$, 
a language model (LM) parameterzied by $\theta$ assigns probability $p_\theta(\xv)$ using the chain rule \citep{bengio2000neural} as follows:
\begin{align*}
    p_\theta(\xv) 
    &= \prod_{t=1}^T p_\theta (x_t \mid {x_1, \dots, x_{t-1}}) 
    \triangleq \prod_{t=1}^T p_\theta (x_t \mid \xv_{<t}),
\end{align*}
where $p_\theta(x_t \mid \xv_{<t} )$ is the predicted distribution of the next token $x_t$ given previous $\xv_{<t} \triangleq (x_1, \dots, x_{t-1})$. 
To generate $\xv$, the LM iteratively computes $p_\theta(x_t \mid \xv_{<t})$ and draws $x_t$ from it; then the sampled $x_t$ is fed back into the LM as part of the inputs for future steps. 
The generation completes if a pre-specified special token that marks the end of the sentence is returned, or the maximum length is reached. 
Different strategies to draw $x_t$ have also been proposed \citep{fan2018hierarchical,holtzman2019curious} to improve the overall fluency and diversity of generated $\xv$. We refer interested readers to \citep{zhao2023survey,min2023recent} and references therein for more detailed backgrounds of LMs. 
From now on, we use $\xv \sim p_\theta(\xv)$ to denote sentence $\xv$ generated by the LM where $p_\theta$ is the sampling distribution parameterized by the LM. 
We treat $\xv$ as a random variable.

\subsection{Debiasing Formulation}

Given sentence $\xv \sim p_\theta(\xv)$, 
we further assume that the sentence $\xv$ contains some global property $g \triangleq g(\xv)$ and mentions some demographic group $a \triangleq a(\xv)$.
Representative global property includes sentiment \citep{huang2019reducing}, regard \citep{sheng2019woman}, and toxicity \citep{gehman2020realtoxicityprompts}.
We follow previous works \citep{yang2023unified, wang2023toward} and treat $a$ as another (global) property of $\xv$.
Built upon the definition of $g$ and $a$, 
an \textit{unbiased} LM should generate $g$ and $a$ independently.
Mathematically, this concept resembles the \textit{demographic parity} fairness notion \citep{pedreshi2008discrimination, huang2019reducing}, whose violation can be quantified by a positive mutual information (MI) $I(g; a) > 0$ \citep{zemel2013learning}. 
To mitigate the bias in the LM, we alter $\xv \sim p_\theta(\xv)$ by solving
\begin{align}\label{eq:debias}
    \min\nolimits_\theta I(g(\xv); a(\xv)), \quad \xv \sim p_\theta (\xv). 
\end{align}

The aforementioned bias is known as \textit{global bias} \citep{sheng2020towards} and has been identified in various LMs \citep{huang2019reducing,ranaldi2023trip}.
As remedies, different interventions have been proposed \citep{liang2021towards,liu2021dexperts,yang2023unified}. 
However, most existing machinery enforced the LMs to output constant $g$ or $a$ in lieu of solving Eqn \eqref{eq:debias} exactly. For instance, 
if $a(\xv) = b$ for all possible $\xv$ where $b$ is a constant, then $I(g(\xv); a(\xv)) = I(g(\xv); b) = 0$ holds and is minimized naturally by the non-negativity of MI \citep{cover1999elements}. 

Nonetheless, 
these 
constant constraints hinder a LM from producing property- or demographic group-specific texts, 
inevitably degrading its generation quality. 
As a consequence, a fairness-fluency trade-off occurs after debiasing \citep{liang2021towards}.
This raises a crucial yet still open question: 

\textit{With elaborated tuning, are trade-offs from the constant constraints able to achieve the Pareto optimality%
\footnote{A Pareto optimal solution refers to a situation where neither fluency nor fairness can be improved without hurting the other. 
We ask if constant constraints are possible to achieve such optimality.}
?}

In this paper, we give a negative answer through a formal information-theoretic analysis.
Subsequently, a more adaptable debiasing method called LIDAO is derived.

\subsection{Limit Intervention for Debiasing Across its Options}

In this section, 
we show that the aforementioned constant requirement 
is an overkill for Eqn \eqref{eq:debias} by proving that a weaker condition is able to solve it.
Targeting on this weakened condition, we come up with a new solution called LIDAO
that is capable of achieving a better fairness-fluency trade-off. 
In particular,
below theorem gives a sufficient condition to minimize $I(g; a)$,
with its proof deferred to App \ref{app:proof:thm:lidao}.

\begin{theorem}\label{thm:lidao}
For sentence $\xv = (x_1, \dots, x_T)$ generated by a LM that involves some global property $g$ and demographic group $a$. If at every step $t > 1$, condition 
\begin{align}\label{eq:igx}
    \ell(g; t) &\triangleq I(g; x_t \mid \xv_{<t} ) = 0, \\
    \text{or}\quad \ell(a; t) &\triangleq I(a; x_t \mid \xv_{<t}) = 0 \label{eq:iax}
\end{align}
holds, then $I(g; a) = 0$. 
In words, if each $x_t$ is relevant to only $g$ or $a$, then $g$ and $a$ are independent with each other.
\end{theorem}



\textbf{Overkill of the constant constraints.}
Thm \ref{thm:lidao} boils downs to that in order to minimize $I(g, a)$, each token $x_t$ needs to be independent of $a$ \textit{only if} it reflects $g$, and vice versa. 
This requirement is much weaker than the constant constraint, as the LM is allowed to generate property- or demographic group-specific texts. 
In contrast, a constant constraint such as $a(\xv) = b$ asks for independence between $a$ and \textit{all} tokens, regardless of whether they are relevant to $g$ or not.

Following Thm \ref{thm:lidao}, 
we propose to 
\textbf{\underline{l}imit the 
\underline{i}ntervention for 
\underline{d}ebiasing 
\underline{a}cross its 
\underline{o}ptions} (LIDAO).
Conceptually, 
LIDAO \textit{chooses to solve} $\ell(g; t)$ or $\ell(a; t)$ adaptively, and we propose two variants to automate this process. 
Specifically, the 
\textbf{min-based LIDAO} seeks to solve the milder $\min \left(\ell(g; t), \ell(a; t)\right)$.
The \textbf{product-based LIDAO}, on the other hand, minimizes 
the product $\ell(g; t) \ell(a; t)$. 
The underlying rationale is that by 
taking gradient 
$\nabla \left( \ell(g; t) \ell(a; t)\right) = \ell(a; t) \nabla \ell(g; t) + \ell(g; t) \nabla \ell(a; t)$,
each loss is weighted by the magnitude of the other. 
Consequently, the smaller term is minimized more aggressively while the larger term tends to remain unchanged. 
Finally, 
we emphasize that LIDAO is a general framework and admits other designs for determining which loss to optimize.

\subsection{LIDAO Implementation}
\label{sec:method:implement}

LIDAO requires computing $\ell(g; t)$ and $\ell(a; t)$ by estimating terms $I(g; x_t \mid \xv_{<t})$ and $I(a; x_t \mid \xv_{<t})$, which proves to be a challenging task due to the lack of analytic form of MI and necessities some approximations thereof \citep{poole2019variational}.
In this work we follow UDDIA \citep{yang2023unified}, a state-of-the-art framework that seeks to minimize $\ell(g; t)$ and $\ell(a; t)$ simultaneously, for such approximate proxies.

\textbf{Approximate Proxies.}
\citet{yang2023unified} proved that the PPLM loss \citep{dathathri2019plug} is a valid proxy for $\ell(g; t)$ under certain condition and we adopt it as $\hat \ell(g; t)$.
For $\ell(a; t)$, \citet{yang2023unified} estimated it by 
\begin{align*}
    \hat \ell(a; t) 
    =&
    \E_{p_\theta(x_t \mid \xv_{<t})} \left[ D \left(q(a \mid \xv_{<t+1}) \| q(a \mid x_t)\right) \right] \notag\\
    &+ \E_{p(a)} \E_{p_\theta(x_t)} \left[ q (a \mid x_t) \log \frac{q(a \mid x_t)}{p(a)} \right], 
\end{align*}
where $p(a)$ is a pre-specified prior distribution over $a$. 
Here $q(a = k \mid x_t) \propto \cos(e(x_t), a_{k, v})$ where $e(x_t)$ is the embedding of $x_t$, 
$a_{k, v}$ is the first principle component of a set of manually-collected seed words $\mathcal{V}_{a=k}
$ from demographic group $a=k$, 
and $\cos(\cdot, \cdot)$ computes the cosine similarity between the two vectors. 
Similarly $q(a = k \mid \xv_{<t+1})$ is constructed by replacing $e(x_t)$ with $(\sum_{i=1}^t e(x_i)) / t$. 
Given these terms, $\hat \ell(a; t)$ can be computed analytically as all expectations are over the next-token distribution.



\textbf{Adjusted min-based LIDAO.}
In execution,
the scales of $\hat \ell(g; t)$ and $\hat \ell(a; t)$ can differ significantly. 
As a consequence, directly plugging them into the min-based principle will make it degrade to choose the same option throughout the generation process. 
To solve this issue, we rescale each loss before determining the minimum. 
These rescaling factors are weighted pooling applied to the history losses. 
Namely, the factor $w(g; t)$ to rescale $\hat \ell(g; t)$ at step $t$ is 
\begin{align*}
    w(g; t) = \left( \sum_{j=1}^{t-1} \gamma^{j} \right)^{-1} \left( \sum_{j=1}^{t-1} \hat \ell(g; t-j) \gamma^j \right),
\end{align*}
and factor $w(a; t)$ is defined similarly. 
Hyper-parameter $\gamma$ controls the speed of forgetting earlier histories
{and a smaller $\gamma$ forgets early histories faster.} In practice we set $\gamma = 0.5$.
Equipped with these factors, the min-based principle solves $\min(\hat \ell(g; t) / w(g; t), \hat \ell(a; t) / w(a; t))$ at step $t$. 

\textbf{Decoding-Time Intervention.}
LIDAO intervenes the generation as in \citet{yang2023unified} by applying bias-tuning \citep{zaken2021bitfit} to $\theta$ before generating $x_t$.
Denote the updated parameters at step $t$ by $\theta_t$ and corresponding distribution by $p_{\theta_t}$, we generate $x_t \sim p_\text{mix}$ by mixing the two distributions 
\begin{align}
    p_{\text{mix}}(x_t \mid \xv_{<t})  
    \propto \left( p_{\theta_t} (x_t \mid \xv_{<t}) \right)^\tau \left( p_{\theta} (x_t \mid \xv_{<t}) \right)^{1-\tau}, \notag
\end{align}
as in \citet{dathathri2019plug}, where $\tau$ is a hyper-parameter. After sampling $x_t$, the parameters in LM is reset to $\theta$. 




\begin{remark}
In this work we adopt the estimators to minimize $\ell(g; t)$ and $\ell(a; t)$ from UDDIA given their good performance. However, LIDAO does not adhere to these choices and one can propose and derive their own methods.
Moreover, the main difference between LIDAO and UDDIA lies in that UDDIA is a constant constraint-based method which seeks to minimize $\ell(g; t)$ and $\ell(a; t)$ at every step, thereof does not offer the same flexibility of LIDAO.
\end{remark}

\section{Robustness against Adversarial Prompts}
\label{sec:adversarial}

As presented in Sec \ref{sec:method}, 
LIDAO
is capable of removing any undesired association between $g$ and demographic group $a$ that may raise fairness concerns in a text $\xv$ generated by the LM. 
Notwithstanding, given a carefully-designed adversarial prompt $\cv$, 
LMs, especially those exhibit strong instruction-following abilities, 
are prone to follow $\cv$ and generate offensive $\xv$ \citep{wei2023jailbroken, wang2023decodingtrust}.
Moreover, when $\cv$ is already {concluding} and specific about some demographic group, $\xv$ may be short and solely consists $g$. 
As such, Eqn \eqref{eq:debias} holds even if $g$ in $\xv$ is systematically biased against the demographic group in $\cv$
due to the constant $a$. 

\citet{wei2023jailbroken}
showed that such threats are ubiquitous on large LMs (LLMs), necessitating the need of robustifying our debiasing formulation as well. 
To this end, 
we note that the vulnerability stems from the underlying assumption that a biased $\xv$ should contain both $a$ and $g$ simultaneously. 
We thereby relax this requirement and extend $a$ as a joint property of $\cv$ and $\xv$ to debias the LM as follows:
\begin{align}\label{eq:debias-ext}
    \min\nolimits_\theta I(g(\xv); a(\cv, \xv)), \quad \xv \sim p_\theta (\xv \mid \cv).
\end{align}
Unfortunately, LIDAO cannot solve Eqn \eqref{eq:debias-ext} provably
due to the additional dependency of $a$ on $\cv$ that is out of control.
Nonetheless, it is viable to quantify the violation as follows.
\begin{proposition}\label{thm:lidao-ext}
For sentence $\xv = (x_1, \dots, x_T)$ generated by a LM that is prompted by $\cv$, i.e., $\xv \sim p_\theta(\xv \mid \cv)$, then
\begin{align}
    I(g(\xv); a(\cv, \xv)) 
    \leq
    I(g(\xv); a(\xv)) \notag 
    &+ H(a(\cv, \xv) \mid a( \xv)), 
\end{align}
where $a(\cv, \xv)$ and $a(\xv)$ denote the referred demographic group in the joint of $(\cv, \xv)$ and $\xv$, respectively. 
\end{proposition}
%
%
%
%
%
The proof is deferred to App \ref{app:proof:thm:lidao-ext}. 
Conceptually,
Prop \ref{thm:lidao-ext} states that
$I(g(\xv); a(\cv, \xv))$ can be bounded from two parts: $I(g(\xv); a(\xv)) $ that can be minimized with LIDAO, 
and $H(a(\cv, \xv) \mid a(\xv))$ as a measure of how easy to infer $a(\cv, \xv)$ from $\xv$ without seeing $\cv$. 
To minimize the second term, we propose a heuristic to encourage $a(\xv) = a(\cv; \xv)$ and refer to this extended LIDAO as {eLIDAO}. 

Inspired by \citet{schick2021self}, eLIDAO resorts to some reference LM that can generate $\xv$ with high affinity $a(\xv) = a(\cv; \xv)$ to guide generations.
Specifically, assume the reference LM predicts $p_{\text{ref}}(x_t \mid \cv, \xv_{<t})$, 
we sample $x_t$ from 
\begin{align*}
    p_\text{mix}^{\text{ext}} (x_t \mid \cv, \xv_{<t}) \propto 
    \begin{cases}
    p_\text{ref} (x_t = v \mid \cv, \xv_{<t}) & \text{if $v \in \mathcal{V}_a$}; \\[1ex]
    p_\text{mix} (x_t = v \mid \cv, \xv_{<t}) & \text{otherwise}.
    \end{cases}
\end{align*}
Here $\mathcal{V}_a$ is a set of seed words whereby $a$ can be identified. 
For simplicity we adopt the ones to construct $\hat \ell(a; t)$ as $\mathcal{V}_a  = \cup_{k=1}^K \mathcal{V}_{a=k}$.
All other implementations is same as in LIDAO. 
We summarize (e)LIDAO in Alg \ref{alg:lidao}.

\textbf{Design of the Reference Model.}
We utilize the instruction-following ability to let the untuned LM itself play the reference model. 
This design is extremely lightweight and requires no extra training or finetuning stage.
To minimize the unintentional influence on the reference model, we use a zero-shot instruction that only asks the LM to focus on the demographic group if it presents in $\cv$ without providing any example. We provide these instructions in App \ref{app:ifrw}. 
It is note worthy that no extensive prompt engineering is conducted, and the sample instructions
can be sub-optimal.

\begin{remark}
Reference models have been used to guide generations \cite{krause2021gedi, liu2021dexperts, schick2021self}. 
However, previous works often need to train the reference model from scratch \citep{krause2021gedi} or fine-tune it with carefully curated data \citep{liu2021dexperts}, which can be time and resource consuming. 
Recently, \citet{schick2021self} successfully conducted self-guiding based on a LM's internal knowledge of \textit{social bias},
but it is still an open question how much such a complicate knowledge exists in the LM and in which way it can be {effectively} and reliably acquired \citep{salinas2023not,cohen2023crawling}. 
In contrast, our solution only requires the knowledge of \textit{demographic groups}, which is expected easier considering that contextual word embedding are demographic-aware \citep{zhao2019gender,sweeney2019transparent}.
\end{remark}

\begin{algorithm}[!t]
\caption{(e)LIDAO algorithm}
\label{alg:lidao}
\begin{algorithmic}[1]

    \STATE \textbf{Input:} 
    Maximum length of generation $T$ (Sec \ref{sec:method}).
    Estimator $\hat \ell(g; t)$ and $\hat \ell(a; t)$ (Sec \ref{sec:method:implement}). 
    Mixing parameter $\tau$ and decaying parameter $\gamma$ (Sec \ref{sec:method:implement}).
    Language model parameterized by $\theta$; a small subset of $\theta$ of will be updated by bias-tuning (Sec \ref{sec:method:implement}). 
    Optional user-specified prompt $\cv$ that is possibly adversarial (Sec \ref{sec:adversarial}).
    \STATE \textbf{Additional input for eLIDAO: }
    Seed words $\mathcal{V}_a$ and eLIDAO system prompt $\cv_{e}$ (Sec \ref{sec:adversarial}).
    \STATE \textbf{Initialize:} Set $\cv = \texttt{<BOS>}$ if not provided. 
    \FOR{$t=1, \dots, T$} 
        \STATE 
        Compute base $p_\theta(x_t \mid \cv, \xv_{<t})$ and reference distribution $p_\theta(x_t \mid \cv_e, \cv, \xv_{t})$. 
        \STATE 
        Compute $\hat \ell(g; t)$ and $\hat \ell(a; t)$ based on $p_\theta(x_t \mid \cv, \xv_{<t})$.
        \STATE 
        Compute the (e)LIDAO loss $\ell_t(\theta)$ by
        \begin{align*}
            \ell_t(\theta) = 
            \min(\hat \ell(g; t) / w(g; t), \hat \ell(a; t) / w(a; t)),
        \end{align*}
        where $w(g; t)$ and $w(a; t)$ are decaying weights as defined in Sec \ref{sec:method:implement}
        if min-based (e)LIDAO is applied, else prod-based (e)LIDAO uses 
        \begin{align*}
        \ell_t(\theta) = \hat \ell(g; t)  \hat \ell(a; t).
        \end{align*}
        \STATE 
        Apply bias-tuning to minimize $\ell_t(\theta)$ by updating $\theta$ with Adam optimizer.
        Obtain debiased parameters $\theta_t$. Compute updated distribution $p_{\theta_t}(x_t \mid \cv, \xv_{<t})$. 
        \STATE
        Compute unnormalized mixing distribution 
        \begin{align*}
            p_{\text{mix}}(x_t \mid \cv, \xv_{<t})  
            \propto& \left( p_{\theta_t} (x_t \mid \cv, \xv_{<t}) \right)^\tau \\
            &\quad \left( p_{\theta} (x_t \mid \cv, \xv_{<t}) \right)^{1-\tau}.
        \end{align*}
        \STATE 
        If using eLIDAO, for all words $v \in \mathcal V_a$, update
        \begin{align*}
            p_{\text{mix}}(x_t = v \mid \cv, \xv_{<t}) \leftarrow p_\theta(x_t = v \mid \cv_e, \cv, \xv_{t}). 
        \end{align*}
        \STATE
        Normalize $p_{\text{mix}}$ and sample token $x_t$. Reset $\theta$. 
        \IF{$x_t = \texttt{<EOS>}$}
        \STATE Break
        \ENDIF
    \ENDFOR 
    \OUTPUT Generated text $\xv = (\xv_1, \dots, \xv_T)$. 
\end{algorithmic}
\end{algorithm}

%% file: subfiles/4_experiment.tex
\section{Experiments}
\label{sec:experiment}

In this section we experiment the proposed (e)LIDAO with debiasing three LMs ranging from 0.7B to 7B parameters on three tasks. 
As expected, (e)LIDAO achieved a much better empirical fairness-fluency trade-off, outperforming existing methods to a large extent. 
Meanwhile, all three LMs including LLMs exhibit substantial bias before intervention, highlighting the necessity for the mitigation \citep{wang2023decodingtrust, ranaldi2023trip}.

\subsection{Experiment Setup}

\textbf{Base Models.}
We experiment with three representative LMs across different model families and sizes. 
\textbf{GPT-2 Large} \citep{radford2019language} has been widely used as a test bed for debiasing methods \citep{dathathri2019plug, krause2021gedi,liu2021dexperts, yang2023unified}, 
and we follow this convention. 
\textbf{OPT} \citep{zhang2022opt} and \textbf{Falcon-7B-Instruct} \citep{falcon40b} are two 7B parameters LLMs. 
Notably, Falcon has undergone an instruction tuning. 
We term the three base models as GPT-2, OPT, and Falcon respectively for brevity.

\textbf{Tasks.}
We consider debiasing three global property $g$ that raise fairness concerns.

\setlist{nosep}
\begin{itemize}[noitemsep]
    \item 
    \textbf{Sentiment}: We debias each LM to generate sentences for different demographic groups with similar sentiment as in \citet{huang2019reducing, yang2023unified}. 

    \item
    \textbf{Regard}: \citet{sheng2019woman} noted that sentiment may not identify negative attitude against some demographic groups and proposed to capture such bias with the new notion \textit{regard}. 
    We debias each LM to output equivalent level of regards towards different demographic groups. 

    \item
    \textbf{Toxicity}: \citet{xu2021detoxifying} noted that detoxification, i.e., reducing toxicity in generated texts, often results in new bias. We follow \citet{yang2023unified} and minimize this bias. 
    
\end{itemize}

\textbf{Baselines.}
We compare the proposed LIDAO with four recent debiasing baselines.

\setlist{nosep}
\begin{itemize}[noitemsep]
    \item 
    \textbf{GeDi} \citep{krause2021gedi} 
    incorporates a reference class-conditioned LM to predict property $g$ for all possible next tokens based on the Bayes rule. 
    We adopted the pretrained GeDi released by the authors and recommended generation hyper-parameters directly. 

    \item 
    \textbf{PPLM} \citep{dathathri2019plug}
    uses the gradient from a lightweight classifier that predicts property $g$ from the LM's hidden representations to remove their associations. 
    We trained linear PPLM classifiers for each LM separately. 
    Due to the lack of regard-annotated dataset, we labeled SST-5 \citep{socher2013recursive} with the regard scorer \citep{sheng2020towards} to trained PPLM. 

    \item
    \textbf{DExperts} \citep{liu2021dexperts} 
    uses two reference LMs finetuned on data at different-leveled property $g$ to adjust the next-token probability.
    We adopted the pretrained DExperts released by the authors and recommended generation hyper-parameters directly. 

    \item
    \textbf{UDDIA} \citep{yang2023unified} 
    seeks to minimize $\ell(g; t)$ in Eqn \eqref{eq:igx} and $\ell(a; t)$ in Eqn \eqref{eq:iax} simultaneously 
    with $\hat \ell(g; t)$ and $\hat \ell(a; t)$ detailed in Sec \ref{sec:method:implement}.
    We dropped the \textit{redo} mechanism in UDDIA for simplification%
    \footnote{We found in our scenarios \textit{redo} often increases the running time with marginal improvement given proper learning rate.}.
    We tuned hyper-parameters for the used estimators on UDDIA and applied them to LIDAO.%
    \footnote{This setup put our proposed method in an unfavorable situation. Nevertheless, ours can still achieve comparable or better performance than that of UDDIA.}.

\end{itemize}

GeDi and DExperts use pre-trained reference models that take GPT-2 as backbones, 
and cannot guide OPT and Falcon that use different tokenization methods without substantial reimplementation. 
Therefore, we only test GeDi and DExperts on debiasing GPT-2 as in the literature. 
Due to page limitations we defer more implementation details to App \ref{app:implementation}.

\textbf{Dataset.}
Following previous works \citep{liu2021dexperts, wang2023toward, yang2023unified}, we focus on the \textit{male} and \textit{female} gender groups.
We adopt the notion of \textit{gender polarity} \citep{dhamala2021bold, wang2023toward} 
and label $a$ as the majority group of gender words.
However, our framework is general and applies to other demographic groups such as \textit{race} as well. 
We focus on a set of paired adversarial prompts released by \citet{yang2023unified} that encourages \textit{toxic} and \textit{biased} texts because of high quality and challenging nature. 
This dataset consist of 175 pairs that are handcrafted from a 1K subset of the RealToxicityPrompts \citep{gehman2020realtoxicityprompts}. 
For each prompt, we generate 10 (GPT-2) or 5 (OPT and Falcon) continuations of maximum length 20, 
which constitute the final dataset we use for evaluation.

\textbf{Evaluation.}
The performance is evaluated from three aspects. 
First, we check how they can manipulate the global property by reporting the averaged $g$. 
This allows us to quantify how in general harmful generations can be avoided, which is pivotal for the evaluating toxicity. 
Second, we evaluate how the fluency is affected by checking the mean perplexity according to a pre-trained GPT-2 XL as in previous works \citep{pozzobon2023goodtriever,yang2023unified}. 
We report perplexity evaluated by larger models from the same family (i.e., OPT-13B and Falcon-40B) in App \ref{app:results}. 
Finally, biases are quantified in a task-specific manner as below to follow the convention in the literature \citep{huang2019reducing, gehman2020realtoxicityprompts, liu2021dexperts, yang2023unified}.

\setlist{nosep}
\begin{itemize}[noitemsep]
    \item 
    \textbf{Sentiment}: the bias is the difference in \textit{the average sentiment score} on two gender groups.

    \item
    \textbf{Regard}: the bias is the \textit{total variation (TV) distance between the 3-way (negative, neutral, and positive) regard distributions} on two gender groups. 

    \item
    \textbf{Toxicity}: the bias is the difference in \textit{the maximum toxicity score} from Perspective API\footnote{\url{https://github.com/conversationai/perspectiveapi}. } on two gender groups as a worst-case measure.
\end{itemize}
For more comprehensive comparison, we analyze the performance with respect to generated $a(\xv)$ and joint $a(\cv, \xv)$ as detailed in Sec \ref{sec:exp:gen} and Sec \ref{sec:exp:joint} respectively. 
Human evaluation results and case studies are also provided.

\subsection{Debias over Generated $a(\xv)$}
\label{sec:exp:gen}

In this section
we assess the bias in sentiment, regard, and toxicity against gender group $a(\xv)$ in generated text $\xv$ and how it can be minimized on the three tasks.
Table \ref{tab:gen} reports these results. 
On the debiasing regard task, 
we combine harmless \textit{neutral} and \textit{positive} regard groups into \textit{non-negative} regard (NNRe) group for straightforward quantification of the averaged regard level. 
We note that 
\textit{on GPT-2 and OPT without instruct tuning, the min-based (e)LIDAO consistently performed better,
whereas Falcon that has undergone instruction tuning favored the prod-based (e)LIDAO.} 
Due to space limitation, Table \ref{tab:gen} reports only these favored results, remaining results are deferred to App \ref{app:results}.


\textbf{Bias in LMs.} 
As per Table \ref{tab:gen}, all three LMs generations contain substantial bias. 
Notably, given the adversarial nature of the prompts, the two LLMs tend to produce more offensive (OPT) or biased (Falcon) texts compared to the smaller GPT-2. 
In specific,
OPT consistently produces the least friendly content, with the lowest sentiment and regard, or the highest toxicity scores. 
Simultaneously, its exhibits the highest bias in regard among the three LMs. 
Falcon, on the other hand, in spite of achieving the best scores, its biases in sentiment and toxicity are also significantly higher than the others.
Meanwhile, by examining the sentiment and regard scores, we observe that although Falcon attains the highest average sentiment score, its regard is the lowest. 
This finding aligns with \citet{sheng2019woman}, confirming that sentiment may not provide a comprehensive measure of the LM's attitude. 
In conclusion, these results highlight that 
\textit{LLMs, regardless of undergoing instruction tuning or not, are prone to generate texts as unfriendly, if not more so, as classical LMs in challenging scenarios,} 
emphasizing the need for additional bias mitigation \citep{wang2023decodingtrust, ranaldi2023trip}.

\textbf{Debiasing Performance.} 
According to Table \ref{tab:gen}, 
when debiasing the three LMs of different sizes, the proposed LIDAO demonstrated remarkable debiasing performance and reached the top two positions in seven out of nine scenarios.
At the same time, LIDAO generated the top two fluent texts in six scenarios based on the perplexity.
In the remaining two of three cases, its perplexity was slightly higher than that of eLIDAO, its own extended version. 
In contrast, the bias and fluency of UDDIA are both worse than LIDAO, despite their implementation similarity.
These empirical findings, coupled with the theoretical guarantee from Thm \ref{thm:lidao}, clearly underscore the significant success of LIDAO in advancing the fairness-fluency trade-off in debiasing LM by a considerable margin compared to existing approaches. 

Finally, we note that although DExperts guided GPT-2 to generate the most fluent texts, these texts exhibited the most negative sentiment and highest toxicity. This highlights the challenge of adapting reference models to new scenarios that diverge from the training corpora \citep{pozzobon2023goodtriever}.

\begin{table*}[htb!] 
\centering
\caption{
Performance of debiasing the three LMs (shadowed in different colors) with respect to generation $a(\xv)$. 
Bias terms are multiplied by 100 for better comparison. https://experience.graebel.com/login
The best debiased and fluent results are in \b{bold} and the second best results are \bb{underlined}.
GPT-2 and OPT used the min-based (e)LIDAO and Falcon used the prod-based (e)LIDAO. 
}
\resizebox{0.75\linewidth}{!}{
\begin{tabular}{@{}l ccc c ccc c ccc@{}}
\toprule
                           & \multicolumn{3}{c}{Sentiment}                                      & & \multicolumn{3}{c}{Regard}                                          & & \multicolumn{3}{c}{Toxicity}\\ 
\cmidrule{2-4}   \cmidrule{6-8} \cmidrule{9-12}
                           
                           & Sent ($\uparrow$)  & Bias ($\downarrow$)   & PPL ($\downarrow$)    & & NNRe ($\uparrow$)   & Bias ($\downarrow$)   & PPL ($\downarrow$)    & & Toxi ($\downarrow$)    & Bias ($\downarrow$)   & PPL ($\downarrow$)\\ 
\midrule
\rowcolor[HTML]{F1F7FF}
GPT-2                      & 0.36$\pm$0.44      & 4.41                  & 22.17                 & & 0.71$\pm$0.55       & 3.89                  & 22.03                 & & 0.16$\pm$0.17          & 9.35                  & 22.17\\ 
\cdashline{1-12}\noalign{\vskip 0.5ex}
\rowcolor[HTML]{F1F7FF}
+GeDi                      & 0.64$\pm$0.44      & 4.99                  & 48.11                 & & -                   & -                     & -                     & & 0.10$\pm$0.12          & 6.81                  & 48.82\\
\rowcolor[HTML]{F1F7FF} 
+DExperts                  & 0.37$\pm$0.45      & 1.42                  & \b{16.00}             & & -                   & -                     & -                     & & 0.15$\pm$0.16          & 8.58                  & \b{16.00}\\
\rowcolor[HTML]{F1F7FF} 
+PPLM                      & 0.38$\pm$0.45      & 1.38                  & 34.10                 & & 0.73$\pm$0.55       & 8.74                  & 36.01                 & & 0.15$\pm$0.15          & \bb{6.24}             & \bb{26.85}\\
\rowcolor[HTML]{F1F7FF} 
+UDDIA                     & 0.41$\pm$0.45      & 2.77                  & 32.32                 & & 0.72$\pm$0.57       & 2.66                  & 34.24                 & & 0.14$\pm$0.15          & 11.74                 & 33.29\\ 
\cdashline{1-12}\noalign{\vskip 0.5ex}
\rowcolor[HTML]{F1F7FF} 
+LIDAO                     & 0.40$\pm$0.45      & \bb{0.32}             & \bb{27.49}            & & 0.69$\pm$0.55       & \bb{2.15}             & \b{28.02}             & & 0.14$\pm$0.15          & 9.88                  & 31.45\\
\rowcolor[HTML]{F1F7FF} 
+eLIDAO                    & 0.40$\pm$0.45      & \b{0.21}              & 28.26                 & & 0.72$\pm$0.55       & \b{1.53}              & \bb{28.08}            & & 0.14$\pm$0.15          & \b{5.91}              & {30.20}\\ 
\midrule

\rowcolor[HTML]{F1FFF1}
OPT                        & 0.33$\pm$0.43      & 0.98                  & 23.62                 & & 0.67$\pm$0.58       & 11.58                 & 23.63                 & & 0.19$\pm$0.20          & 0.78                  & 23.62\\ 
\cdashline{1-12}\noalign{\vskip 0.5ex}
\rowcolor[HTML]{F1FFF1} 
+PPLM                      & 0.40$\pm$0.46      & 4.80                  & 49.07                 & & 0.72$\pm$0.63       & 5.10                  & 46.86                 & & 0.17$\pm$0.18          & 2.27                  & \bb{36.69}\\
\rowcolor[HTML]{F1FFF1} 
+UDDIA                     & 0.39$\pm$0.45      & 5.40                  & 62.78                 & & 0.74$\pm$0.63       & \bb{1.81}             & 66.39                 & & 0.16$\pm$0.17          & \bb{1.69}             & 39.79\\ 
\cdashline{1-12}\noalign{\vskip 0.5ex}
\rowcolor[HTML]{F1FFF1} 
+LIDAO                     & 0.43$\pm$0.46      & \b{0.03}              & \bb{40.16}            & & 0.78$\pm$0.62       & \b{1.09}              & \bb{42.06}             & & 0.16$\pm$0.17         & \b{0.00}              & {37.17}\\
\rowcolor[HTML]{F1FFF1}
+eLIDAO                    & 0.39$\pm$0.45      & \bb{2.40}             & \b{38.69}             & & 0.71$\pm$0.66       & 3.91                  & \b{34.85}            & & 0.15$\pm$0.17           & 9.10                  & \b{35.70}\\ 
\midrule

\rowcolor[HTML]{FFF7EE}
Falcon                     & 0.41$\pm$0.46      & 6.39                  & 27.40                 & & 0.65$\pm$0.61       & 2.83                  & 27.00                 & & 0.15$\pm$0.16          & 12.64                 & 27.40\\ 
\cdashline{1-12}\noalign{\vskip 0.5ex}
\rowcolor[HTML]{FFF7EE} 
+PPLM                      & 0.46$\pm$0.46      & \b{2.12}              & 70.17                 & & 0.72$\pm$0.68       & 8.16                  & 71.51                 & & 0.13$\pm$0.14          & 21.79                 & 57.83\\
\rowcolor[HTML]{FFF7EE} 
+UDDIA                     & 0.47$\pm$0.46      & 4.54                  & \b{29.03}             & & 0.77$\pm$0.66       & \b{2.56}              & \bb{30.40}            & & 0.12$\pm$0.14          & 14.94                 & 29.26\\ 
\cdashline{1-12}\noalign{\vskip 0.5ex}
\rowcolor[HTML]{FFF7EE} 
+LIDAO                     & 0.41$\pm$0.46      & {2.96}                & 30.87                 & & 0.70$\pm$0.65       & \bb{2.59}             & \b{28.49}             & & 0.13$\pm$0.14          & \bb{10.16}            & \bb{28.01}\\
\rowcolor[HTML]{FFF7EE} 
+eLIDAO                    & 0.44$\pm$0.46      & \bb{2.74}             & \bb{30.51}            & & 0.75$\pm$0.53       & 7.61                  & 38.45                 & & 0.14$\pm$0.15          & \b{3.95}              & \b{26.92}\\
\bottomrule

\end{tabular}
}

\label{tab:gen}
\end{table*}

\subsection{Debias over Joint $a(\cv; \xv)$}

\label{sec:exp:joint}

In this section,
we evaluate the bias with respect to gender group $a(\cv; \xv)$ by considering the generated text $\xv$ together with the prompt context $\cv$ and how different approaches can optimize such bias on three tasks. 
Results are presented in Table \ref{tab:joint}.
As previous we use the non-negative regard (NNRe) to verify how average regard is affected. 
On GPT-2 and OPT we report the min-based (e)LIDAO results and on Falcon we report the prod-based (e)LIDAO results, with remaining results deferred to App \ref{app:results}.

\textbf{Bias in LMs.}
As shown in Table \ref{tab:joint}, we note that all three LMs once again exhibit considerable bias against different gender groups.
Many previously discussed findings remain applicable here so we refrain from reiterating them for brevity and direct readers to the prior section. 
Notably, the bias in regard from OPT and the toxicity from Falcon are now lower, suggesting that certain texts $\xv$ that did not mention $a$ were not accounted for in the formulation in Eqn \eqref{eq:debias}. 
These distinctions highlight the significance of \textit{considering the influence of prompts when identifying and evaluating bias, particularly when dealing with LLMs.}

\textbf{Debiasing Performance.}
Referring to Table \ref{tab:joint}, while eLIDAO lacks the theoretical guarantee of LIDAO for debiasing, it demonstrated strong empirical success in mitigating bias on the three LMs. 
In contrast, LIDAO, lacking consideration for the prompt $\cv$, experienced a considerable decline in debiasing performance
as elucidated by Prop \ref{thm:lidao-ext} due to the presence of $H(a(\cv; \xv) \mid a(\xv))$. 
Compared with LIDAO, eLIDAO consistently achieved lower bias in all nine scenarios, and attained the best two performance in the six. 
In the meantime, eLIDAO maintained compelling, if not the best, fluency. Namely, it generated the top two fluent texts in eight out of nine scenarios after ignoring its original LIDAO version. 
Consequently, eLIDAO attained a notable fairness-fluency trade-off. 
These results collectively established the effectiveness of utilizing the LM itself as the reference model for the self-guidance, as proposed in Sec \ref{sec:adversarial}.

Additionally, DExperts again proved less effective in altering sentiment and toxicity in GPT-2 generations, despite being designed to \citep{liu2021dexperts}. 
We conjecture that this reflects how distribution shift can lead a reference model inaccurately captured the desired properties, resulting in the performance degradation. 
The proposed eLIDAO, on the contrary, only requires a reference model to identify \textit{gender polarity}, which is expected invariant in diverse contexts, making the resultant guidance more robust to domain shifts.

\begin{table*}[htb!] 
\centering
\caption{
Performance of debiasing the three LMs (shadowed in different colors) with respect to joint $a(\cv, \xv)$. 
Bias terms are multiplied by 100 for better comparison. 
The best debiased and fluent results are in \b{bold} and the second best results are \bb{underlined}.
GPT-2 and OPT used the min-based (e)LIDAO and Falcon used the prod-based (e)LIDAO. 
}
\label{tab:joint}
\resizebox{0.75\linewidth}{!}{
\begin{tabular}{@{}l ccc c ccc c ccc@{}}
\toprule
                           & \multicolumn{3}{c}{Sentiment}                                      & & \multicolumn{3}{c}{Regard}                                          & & \multicolumn{3}{c}{Toxicity}\\ 
\cmidrule{2-4}   \cmidrule{6-8} \cmidrule{9-12}
                           
                           & Sent ($\uparrow$)  & Bias ($\downarrow$)   & PPL ($\downarrow$)    & & NNRe ($\uparrow$)   & Bias ($\downarrow$)   & PPL ($\downarrow$)    & & Toxi ($\downarrow$)     & Bias ($\downarrow$)   & PPL ($\downarrow$)\\ 
\midrule
\rowcolor[HTML]{F1F7FF}
GPT-2                      & 0.35$\pm$0.44      & 1.90                  & 23.20                 & & 0.69$\pm$0.56       & 3.66                  & 23.23                 & & 0.14$\pm$0.16           & 16.84                 & 23.20\\ 
\cdashline{1-12}\noalign{\vskip 0.5ex}
\rowcolor[HTML]{F1F7FF}
+GeDi                      & 0.64$\pm$0.44      & 2.21                  & 47.11                 & & -                   & -                     & -                     & & 0.08$\pm$0.11           & 7.12                  & 49.00\\
\rowcolor[HTML]{F1F7FF} 
+DExperts                  & 0.37$\pm$0.45      & \bb{0.68}             & \b{16.55}             & & -                   & -                     & -                     & & 0.14$\pm$0.16           & \b{1.00}              & \b{16.55}\\
\rowcolor[HTML]{F1F7FF} 
+PPLM                      & 0.37$\pm$0.45      & 1.10                  & 37.06                 & & 0.71$\pm$0.55       & 3.13                  & 38.30                 & & 0.13$\pm$0.15           & \bb{1.22}             & \bb{28.02}\\
\rowcolor[HTML]{F1F7FF} 
+UDDIA                     & 0.41$\pm$0.45      & 2.29                  & 33.85                 & & 0.72$\pm$0.56       & 3.46                  & 35.29                 & & 0.12$\pm$0.14           & 11.74                 & 33.43\\ 
\cdashline{1-12}\noalign{\vskip 0.5ex}
\rowcolor[HTML]{F1F7FF} 
+LIDAO                     & 0.38$\pm$0.45      & 2.11                  & \bb{28.43}            & & 0.73$\pm$0.54       & \bb{0.29}             & \bb{30.09}            & & 0.12$\pm$0.14           & 8.28                  & 31.27\\
\rowcolor[HTML]{F1F7FF} 
+eLIDAO                    & 0.38$\pm$0.45      & \b{0.46}              & 29.25                 & & 0.73$\pm$0.55       & \b{0.11}              & \b{29.79}             & & 0.12$\pm$0.14           & 5.91                  & {30.70}\\ 
\midrule

\rowcolor[HTML]{F1FFF1}
OPT                        & 0.34$\pm$0.44      & 1.04                  & 25.52                 & & 0.69$\pm$0.58       & 4.53                  & 25.42                 & & 0.19$\pm$0.20           & 0.78                  & 23.62\\ 
\cdashline{1-12}\noalign{\vskip 0.5ex}
\rowcolor[HTML]{F1FFF1} 
+PPLM                      & 0.40$\pm$0.46      & 4.44                  & 52.10                 & & 0.70$\pm$0.65       & 3.90                  & 47.59                 & & 0.16$\pm$0.18           & \b{1.78}              & \bb{39.50}\\
\rowcolor[HTML]{F1FFF1} 
+UDDIA                     & 0.40$\pm$0.45      & \bb{1.07}             & 63.83                 & & 0.70$\pm$0.61       & \bb{2.19}             & 68.34                 & & 0.15$\pm$0.17           & \bb{2.94}             & 42.64\\ 
\cdashline{1-12}\noalign{\vskip 0.5ex}
\rowcolor[HTML]{F1FFF1} 
+LIDAO                     & 0.40$\pm$0.45      & 2.96                  & \bb{43.31}            & & 0.74$\pm$0.62       & 9.09                  & \bb{45.97}            & & 0.14$\pm$0.17           & 6.29                  & 39.72\\
\rowcolor[HTML]{F1FFF1}
+eLIDAO                    & 0.39$\pm$0.45      & \b{0.83}              & \b{41.38}             & & 0.70$\pm$0.64       & \b{2.16}              & \b{36.50}             & & 0.14$\pm$0.17           & 4.80                  & \b{37.75}\\ 
\midrule

\rowcolor[HTML]{FFF7EE}
Falcon                     & 0.40$\pm$0.45      & 2.28                  & 29.44                 & & 0.70$\pm$0.63       & 1.60                  & 29.24                 & & 0.14$\pm$0.16           & 0.79                  & 29.44\\ 
\cdashline{1-12}\noalign{\vskip 0.5ex}
\rowcolor[HTML]{FFF7EE} 
+PPLM                      & 0.46$\pm$0.46      & \bb{1.41}             & 72.77                 & & 0.72$\pm$0.67       & 7.11                  & 74.72                 & & 0.11$\pm$0.13           & 19.38                 & 61.56\\
\rowcolor[HTML]{FFF7EE} 
+UDDIA                     & 0.45$\pm$0.46      & \b{0.81}              & \bb{31.79}            & & 0.75$\pm$0.66       & \b{1.90}              & \bb{32.68}            & & 0.11$\pm$0.13           & 3.97                  & 32.55\\ 
\cdashline{1-12}\noalign{\vskip 0.5ex}
\rowcolor[HTML]{FFF7EE} 
+LIDAO                     & 0.40$\pm$0.45      & 1.61                  & \b{31.67}             & & 0.73$\pm$0.67       & 3.10                  & \b{30.64}             & & 0.11$\pm$0.13           & \bb{1.30}             & \b{30.42}\\
\rowcolor[HTML]{FFF7EE} 
+eLIDAO                    & 0.43$\pm$0.46      & 1.57                  & 31.90                 & & 0.76$\pm$0.54       & \bb{2.43}             & 39.47                 & & 0.12$\pm$0.15           & \b{1.03}              & \bb{30.99}\\
\bottomrule

\end{tabular}
}

\end{table*}

\subsection{Human Evaluation and Case Studies}

\begin{table*}[htb!] 
\centering
\caption{
Samples from debiasing sentiment (top), regard (middle), and toxicity (bottom) on three LLMs with UDDIA and (e)LIDAO as shadowed in different colors. 
All results are generated from the first prompt for simple illustrations. 
On the undebiased generations, we color the tokens that 
\textbf{should be} independent with {\color{red}{property $g$}} and {\color{blue}{gender $a$}} as suggested by LIDAO, respectively. 
}
\label{tab:samples}

\resizebox{0.75\linewidth}{!}{
\newcommand{\g}[1]{{\color{red}#1}}
\renewcommand{\a}[1]{{\color{blue}#1}}
\begin{tabular}{l}
\toprule
\textbf{Prompt:}
\textit{Instead, these men watched the man they say humiliated }
 \\
\midrule

\textbf{[GPT-2]}

\a{them walk away with his}
\g{wallet} 
\a{-- and one}
\g{man was heckled}
\a{and screamed} 
\g{at}
\a{by the gang} 
\g{after}
...\\

\cdashline{1-1}\noalign{\vskip 0.5ex}

\rowcolor[HTML]{F1F7FF}
\textbf{[UDDIA]}
them on the national television. ``We lost our jobs after that interview,'' one recalls. ``We... \\

\cdashline{1-1}\noalign{\vskip 0.5ex}
\rowcolor[HTML]{F1F7FF}
\textbf{[LIDAO]}
them with his boldness and brash style of play. ``I watched him walk in...\\

\rowcolor[HTML]{F1F7FF}
\textbf{[eLIDAO]}
them, take them to task for their religious practices and then offer them assistance...\\
\midrule

\rowcolor[HTML]{F1FFF1}
\textbf{[OPT]}

\a{them with a}
\g{smile}
\a{and a}
\g{laugh}
\a{while taking their}
\g{own photos.} 
\a{Then when he saw something}...\\

\cdashline{1-1}\noalign{\vskip 0.5ex}

\rowcolor[HTML]{F1FFF1}

\textbf{[UDDIA]}
and tortured them die and medals of honor were to be later awarded for dogged service, this is\\

\cdashline{1-1}\noalign{\vskip 0.5ex}
\rowcolor[HTML]{F1FFF1}

\textbf{[LIDAO]}
and tortured them for years - a man who has been in and out of prison for more than 12...\\

\rowcolor[HTML]{F1FFF1}
\textbf{[eLIDAO]}
and tortured them for years - a man who has now been ordered to stand for his crimes at the...\\

\midrule

\rowcolor[HTML]{FFF7EE} 
\textbf{[Falcon]}

\a{the} 
\g{black male} 
\a{and} 
\g{did} 
\a{nothing. But} 
\g{today} 
\a{, that's} 
\g{not so much the case}... \\

\cdashline{1-1}\noalign{\vskip 0.5ex}

\rowcolor[HTML]{FFF7EE} 
\textbf{[UDDIA]}
them, approached a woman and asked for sex - but after they say it was in self defense...\\

\cdashline{1-1}\noalign{\vskip 0.5ex}

\rowcolor[HTML]{FFF7EE} 
\textbf{[LIDAO]}
them, approached a woman and asked her if he could take a seat next to her in a crowded...\\

\rowcolor[HTML]{FFF7EE} 
\textbf{[eLIDAO]} 
them, approached a woman and asked her if she would like to join them on the dance floor...\\

\bottomrule
\end{tabular}
}

\end{table*}

We further conduct human evaluation and case studies on (e)LIDAO and compare them to the backbone UDDIA.
Due to the substantial human effort, each LM is evaluated on one task. 
We randomly sampled 20\% of prompts, and generated one continuation per prompt. 
This process resulted in 420 pairs to compare in total. 
We requested two annotators to choose which sentience in each pair has better quality of \textit{fluency} and \textit{friendliness}. 
Results in Table \ref{tab:human} again affirm the effectiveness of (e)LIDAO.


We conclude this section with a case study on the generated texts. 
Table \ref{tab:samples} presents the original and debiased generations from three LLMs, one for each task. 
Examples show that the proposed (e)LIDAO is capable of generating high quality texts on par with existing methods. 

To gain a deeper understanding on how LIDAO seeks to debias the generation through limited interventions, 
we color-code each token in the original generations with blue and red to denote whether the property $g$ or gender $a$'s influence \textit{should be} removed in order to debias. 
Namely, we label tokens that achieved lower $\hat \ell(g; t)$ by blue and that achieved lower $\hat \ell(a; t)$ by red. 
We found that the debiasing strategy employed by LIDAO enjoys a certain level of interpretability. 
For instance, 
in Falcon's generation, to eliminate the bias, LIDAO argues that the generation of ``black male'' should not depend on $g$, and the action ``(did) nothing'' should be agnostic with $a$, as these dependencies are mild in nature.

\begin{table}[htb!]
\centering
\caption{
Portions of human preferred generations. 
Each LM is evaluated on one task, respectively. Winning results are in bold.
}
\label{tab:human}

\resizebox{0.95\linewidth}{!}{
\renewcommand{\tabcolsep}{2pt}
\begin{tabular}{lccc c ccc }
\toprule
                                            & \multicolumn{7}{c}{Human Preferrence Portion ($\uparrow$)} \\
\cmidrule{2-8}
                                            & LIDAO         & Tied  & UDDIA     & & eLIDAO      & Tied  & UDDIA\\
\cmidrule{2-4} \cmidrule{6-8} 

GPT-2 (Sent.)                               & \b{0.229}     & 0.571 & 0.2       & & \b{0.357}   & 0.3   & 0.343\\

OPT (Regard)                                & \b{0.357}     & 0.314 & 0.329     & & 0.314       & 0.371 & 0.314\\

Falcon (Detox.)                             & \b{0.3}       & 0.429 & 0.271     & & \b{0.286}   & 0.443 & 0.271\\
\bottomrule
\end{tabular}
}

\end{table}

%% file: subfiles/2_background.tex
\section{Related Works}
\label{sec:background}

\textbf{Bias in text generation} 
as a distribution-level difference across varied demographic groups has been widely observed in the literature \citep{gallegos2023bias}.
Existing measures of social bias falls into two categories. 
\textit{Local bias} \citep{liang2021towards} refers to the difference in the high-dimensional probability distributions of a LM and boils down to the accumulated difference in token-level distributions. 
\textit{Global bias} \citep{sheng2020towards} considers the difference in the distribution of some low-dimensional global property \citep{sheng2021societal}.
In general, the practical harm of global bias is more straightforward than local bias
and has attracted more attention \citep{blodgett2020language}. 
In this work we study the global bias from an information-theoretic perspective.

\textbf{Bias mitigation}
has accumulated a vast literature \citep{li2023survey}. 
Early attempts seek to finetune a pretrained LM on some carefully-curated datasets that contains no or minimal bias \citep{lu2020gender, saunders2020reducing} or through a regularization training \citep{huang2019reducing,peng2020reducing,wang2023toward}. 
However, these methods are often resource-consuming, hindering their practicality for larger LMs. 
Consequently, the decoding-time intervention paradigm is favoured \citep{liu2021dexperts, yang2023unified}. 
Representative decoding-time interventions consist of using reference models to guide generations \citep{liang2020towards, krause2021gedi, liu2021dexperts}, or injecting adversarial triggers to stimulate unbiased generations \citep{sheng2020towards, schick2021self}. 
However, many existing methods lack a formal debiasing formulation, making them largely unprincipled. 
Recent works \citep{yang2023unified, wang2023toward} incorporated mutual information (MI) in their bias mitigation formulation,  
but they only utilized MI to characterize the dependency between the global properties and generated tokens. 
In this work, we propose to formulate the global bias with MI directly. 
Notably, under the proposed formulation, existing methods can be seen as imposing some constant constraints on the LM, steering it to generate texts where every token is agnostic to either the demographic group or the global property, regardless of the necessity to do so.
Taking insights from this observation, we establish a principled way to debias the LM under a sufficient but weaker constraint.
Compared with existing methods, our solution is capable of achieving a better fairness-fluency trade-off both theoretically and empirically.

%% file: subfiles/5_conclusion.tex
\section{Conclusion and Limitations}
\label{sec:conclusion}

We propose LIDAO towards a better fairness-fluency trade-off in debiasing text generation with language models (LMs). 
Through a theoretical analysis, we reveal that existing methods imposed excessive constraints to the LMs during the interventions, hindering them from achieving the optimal trade-offs. 
On the contrary, LIDAO debiases the LMs with only limited interventions, thereby achieves a better trade-off both theoretically and empirically.
We further extend LIDAO to adversarial scenarios, where an adversarial prompt may stimulate large language models (LLMs) to generate harmful text that is biased only if it is situated in the context from the prompt. 
Relying on the instruction-following ability of LLMs, a lightweight heuristic is proposed to make the potential bias explicit within the generated text, whereby LIDAO can be applied for mitigation. 
Throughout extensive experiments, the proposed method outperforms representative baselines in mitigating bias while maintaining high fluency by a large margin.  


One limitation of our work is that LIDAO only reduces, but does not eliminate the bias,
which we presume accounts for the inexact proxies. 
As in prior works, we rely on some mutual information proxies to optimize; 
and the referred gender in a text is also defined on a word frequency base. 
All these approximations can be inaccurate and result in the imperfect results.
We leave exploring more effective mutual information proxies and ways to determine the discriminated demographic groups as our future direction. 
Meanwhile, the evaluation of fluency and global property: sentiment, regard, and toxicity, are algorithmic-based (e.g., with a sentiment-analysis tool and Perspective API). 
However, these definitions can be context-specific and subjective, and the current evaluation paradigm is far from ideal, and it is crucial to develop a better one. 
Second, this work focuses on the gender bias because of its prominent importance in the literature. Due to the lack of off-the-shelf benchmark dataset that contains high quality adversarial prompts to encourage toxic and biased generations against different race and aged groups, we plan to conduct these studies in the next step.
Finally, this work focuses on the mono-lingual text generation using three (L)LMs. 
Given the increasing trend of multi-lingual and multi-cultural real-world applications, we leave the direction of expanding the proposed method for cross-lingual transfer as our future work.

%% file: subfiles/6_appendix.tex
\clearpage
\onecolumn


\section{Omitted Proofs}

In this section we provide the omitted proofs.


\subsection{Proof of Thm \ref{thm:lidao}}
\label{app:proof:thm:lidao}

We start with restating Thm \ref{thm:lidao}. Recall that $\xv_{<t} \triangleq (x_1, \dots, x_{t-1})$. 

\begin{theorem*}
For sentence $\xv = (x_1, \dots, x_T)$ generated by a LM that involves some global property $g$ and demographic group $a$. If at every step $t > 1$, condition 
\begin{align*}
    \ell(g; t) &\triangleq I(g; x_t \mid \xv_{<t} ) = 0, \\
    \text{or}\quad \ell(a; t) &\triangleq I(a; x_t \mid \xv_{<t}) = 0 
\end{align*}
holds, then $I(g; a) = 0$. 
In words, if each $x_t$ is relevant to only $g$ or $a$, then $g$ and $a$ are independent with each other.
\end{theorem*}

The proof of Thm \ref{thm:lidao} relies on the following auxiliary Lemma.

\begin{lemma}\label{lem:lidao}
Follow the notations in Thm \ref{thm:lidao},  the change in mutual information between $g$ and $a$ caused by token $x_t$ at step $t$ is given by 
\begin{align*}
    &I(g; a \mid \xv_{<t+1}) - I(g; a \mid \xv_{<t}) 
    = I(g; x_t \mid \xv_{<t}, a) - I(g; x_t \mid \xv_{<t}).
\end{align*}
\end{lemma}

\begin{proof}[Proof of Lem \ref{lem:lidao}]
The proof follows the chain rule of the mutual information \citep{cover1999elements}
\begin{align*}
    &I(g; a \mid \xv_{<t+1}) \\
    =& 
    I(g; a, \xv_{<t+1}) - I(g; \xv_{<t+1}) \\
    =& 
    I(g; a, \xv_{<t}, x_t) - I(g; \xv_{<t}, x_t) \\
    \overset{(i)}{=}&
    \left\{ I(g; \xv_{<t}) + I(g; a \mid \xv_{<t}) + I(g; x_t \mid \xv_{<t}, a) \right\} \\
    & - \left\{I(g; \xv_{<t}) + I(g; x_t \mid \xv_{<t})\right\} \\
    =&
    I(g; a \mid \xv_{<t}) + I(g; x_t \mid \xv_{<t}, a) - I(g; x_t \mid \xv_{<t}),
\end{align*}

where $(i)$ applies the chain rule to the two terms separately. 
Reorganize the equation completes the proof. 
\end{proof}

With Lem \ref{lem:lidao}, we are ready to prove Thm \ref{thm:lidao}.

\begin{proof}[Proof of Thm \ref{thm:lidao}]
First, we note that 
\begin{align*}
    I(g; a \mid \xv=(x_1, \dots, x_T) ) = 0,
\end{align*}
i.e., when the sentence $\xv$ is determined, its property $g$ and $a$ are no more random and has zero mutual information thereof. 
In addition, if at every step $t$, the mutual information change 
\begin{align*}
    I(g; a \mid \xv_{<t+1}) - I(g; a \mid \xv_{<t}) = 0, 
\end{align*}
then the induction holds
\begin{align*}
    I(g; a) 
    &\overset{(i)}{=}
    I(g; a \mid \xv_{<1}) \\
    &= \dots \\
    &= I(g; a \mid \xv_{<T+1}) \\
    &= 0,
\end{align*}
where $(i)$ holds from the convention that $\xv_{<1} = x_0 = \texttt{<BOS>}$ is a special token to remark the begin of sentence $\xv$ that tells the LM to start generating $\xv$. 

Now, according to Lem \ref{lem:lidao}, it suffices to require 
\begin{align}
    0 =&
    I(g; a \mid \xv_{<t+1}) - I(g; a \mid \xv_{<t}) \notag \\
    =& 
    I(g; x_t \mid \xv_{<t}, a) - I(g; x_t \mid \xv_{<t}) \notag \\
    =&
    I(g, a; x_t \mid \xv_{<t}) - I(a; x_t \mid \xv_{<t}) - I(g; x_t \mid \xv_{<t}). \label{eq:thm:lidao:t}
\end{align}
This condition will be satisfied if the conditional independence $a \perp x_t \mid \xv_{<t}$ or $g \perp x_t \mid \xv_{<t}$ holds.
To see this, note that $g$ and $a$ are symmetric in Eqn \ref{eq:thm:lidao:t}; to avoid cluttering, we take $a$ as an instance and the same logic applies to $g$ directly. 

Now let's assume $a \perp x_t \mid \xv_{<t}$, then by definition 
\begin{align*}
    &I(g, a; x_t \mid \xv_{<t})  \\
    =&
    \E_{p(\xv_{<t})}
    \E_{p(g, a, x_t \mid \xv_{<t})}  \left[ \log \frac{p(g, a , x_t \mid \xv_{<t})}{p(g, a \mid \xv_{<t}) p(x_t \mid \xv_{<t})} \right] \\
    =&
    \E_{p(\xv_{<t})}
    \E_{p(g, a, x_t \mid \xv_{<t})} \left[ \log \frac{p(g, a , x_t \mid \xv_{<t})}{p(g \mid \xv_{<t}) p(a \mid \xv_{<t}) p(x_t \mid \xv_{<t})} \right] \\
    &- \E_{p(\xv_{<t})}
    \E_{p(g, a, x_t \mid \xv_{<t})} \left[ \log \frac{p(g, a \mid \xv_{<t})}{p(g \mid \xv_{<t}) p(a \mid \xv_{<t})} \right] \\
    =&
    I(g; a, x_t \mid \xv_{<t}) - I(g; a \mid \xv_{<t}) \\
    \overset{(i)}{=}&
    \left(I(g; x_t \mid \xv_{<t}) + I(g; a \mid \xv_{<t})\right) - I(g; a \mid \xv_{<t}) \\
    =& I(g; x_t \mid \xv_{<t}),
\end{align*}

where $(i)$ holds from  $a \perp x_t \mid \xv_{<t}$. Thereby Eqn \eqref{eq:thm:lidao:t} reduces to 
\begin{align*}
     &I(g, a; x_t \mid \xv_{<t}) - I(a; x_t \mid \xv_{<t}) - I(g; x_t \mid \xv_{<t}) \\
     =&
     I(g; x_t \mid \xv_{<t}) - I(a; x_t \mid \xv_{<t}) - I(g; x_t \mid \xv_{<t}) \\
     =& 
     0,
\end{align*}
as again $a \perp x_t \mid \xv_{<t}$ and this completes the proof.

\end{proof}

\subsection{Proof of Prop \ref{thm:lidao-ext}}
\label{app:proof:thm:lidao-ext}

We start with restating Prop \ref{thm:lidao-ext}.

\begin{proposition*}
For sentence $\xv = (x_1, \dots, x_T)$ generated by a LM that is prompted  by $\cv$, i.e., $\xv \sim p_\theta(\xv \mid \cv)$, then
\begin{align*}
    I(g(\xv); a(\cv, \xv)) 
    \leq& 
    I(g(\xv); a(\xv))
    + H(a(\cv, \xv) \mid a( \xv)), 
\end{align*}
where $a(\cv, \xv)$ and $a(\xv)$ denote the referred demographic group in the concatenation of $(\cv, \xv)$ and generation $\xv$, respectively. 
\end{proposition*}

\begin{proof}
The proof is largely adopted from \citet{liu2024towards} and relies on the variation of information (VI, \citet{kraskov2005hierarchical}), a well-defined metric that satisfies the triangle inequality. 
Specifically, for random variable $a(\cv, \xv)$ and $a(\xv)$, their VI is defined as 
\begin{align}
    VI(a(\cv, \xv), a(\xv)) 
    \triangleq&
    H(a(\cv, \xv)) + H(a(\xv))
    - 2I(a(\cv, \xv); a(\xv)), \label{eq:vi}
\end{align}
where $H(\cdot)$ is the entropy. 
Apply the triangle inequality to $a(\cv, \xv)$, $a(\xv)$, and $g(\xv)$, we have 
\begin{align*}
    VI(a(\xv), g(\xv)) 
    \leq& 
    VI(a(\cv, \xv), a(\xv)) 
    + VI(a(\cv, \xv), g(\xv)).
\end{align*}

Plug in Eqn \eqref{eq:vi} and reorganize the result, we have 
\begin{align*}
    -I(g(\xv); a(\xv))
    \leq&
    \left\{H(a(\cv, \xv)) - I(a(\cv, \xv); a(\xv))\right\} \\
    &- I(a(\cv, \xv); g(\xv)) \\
    I(a(\cv, \xv); g(\xv))
    \leq& I(g(\xv); a(\xv)) 
    + H(a(\cv, \xv) \mid a( \xv)).
\end{align*}

This completes the poof. 
\end{proof}


\section{Instructions for Encouraging Mentions of Demographic Groups}
\label{app:ifrw}

In this section we present the sample instructions for the three LMs when using them as the reference models, please see Sec \ref{sec:adversarial} for more details. 
Given each instruction, we replace \texttt{\#GENDER\#} with the referred gender and prepend the instruction to the prompt.

\begin{table*}[htb!] 
\centering
\caption{
Templates used for instructing the three LMs to achieve high affinity $a(\xv) = a(\cv; \xv)$. 
We break templates into several lines for the sake of illustration.
}
\label{tab:instruction}

\resizebox{0.8\linewidth}{!}{
\newcommand{\g}[1]{{\color{red}#1}}
\renewcommand{\a}[1]{{\color{blue}#1}}
\begin{tabular}{ll}
\toprule
& \textit{Instruct the LM to mention the demographic group in the generation so that $a(\xv) = a(\cv; \xv)$. }
 \\

\midrule


GPT-2

& 
\makecell[l]{
\texttt{Please help me complete the sentence in \#Input\#. } \\
\texttt{You should focus on the mentioned \#GENDER\# in \#Input\# as frequently as possible.}\\
\texttt{$\backslash$n$\backslash$n\#Input\#:$\backslash$n}
}

\\

\midrule

OPT

& 
\makecell[l]{
\texttt{Please help me complete the sentence in \#Input\#. } \\
\texttt{You should focus on the mentioned \#GENDER\# in \#Input\# as frequently as possible.}\\
\texttt{$\backslash$n$\backslash$n\#Input\#:$\backslash$n}
}

\\

\midrule

Falcon

& 
\makecell[l]{
\texttt{\#Instruction\#:$\backslash$nYour objecive is to continue a given text in \#Input\#.  }\\
\texttt{If a \#GENDER\# presents in \#Input\#, you SHOULD mention this \#GENDER\# as frequently}\\ 
\texttt{and explicitly as possible.}\\
\texttt{$\backslash$n$\backslash$n\#Input\#:$\backslash$n}
}

\\

\bottomrule
\end{tabular}
}

\end{table*}

\section{Implementation Details}
\label{app:implementation}

In this section we provide more implementation details. 
As presented in Sec \ref{sec:experiment}, GeDi, and DExperts used default hyper-parameters. We tuned PPLM for OPT and Falcon. 
OPT takes PPLM updates for 10 steps with learning rate 0.2 per generation step. 
Falone takes PPLM updates for 3 steps with learning rate 0.2 per generation step.

In both UDDIA and (e)LIDAO, we omitted the \textit{redo} mechanism and applied the bias-tuning to the top 18 layers. 
Following \citet{yang2023unified}, we take one gradient descent step with the Adam optimizer \citep{kingma2014adam}. 
Table \ref{tab:hparams} reports detailed hyper-parameters. 
All algorithms used the Nucleus Sampling \citep{holtzman2019curious}.  
Note that sampling parameters are determined on a model basis.

Finally, we follow previous works and sanitize generation by filtering out texts that have perplexity larger than 200.

\begin{table*}[htb!] 
\centering
\caption{
Hyper-parameters used in our experiment. On GPT-2 only learning rate (marked with $*$) was tuned. All other choices were adopted from \citet{yang2023unified}. 
}
\label{tab:hparams}
\resizebox{0.9\linewidth}{!}{
\begin{tabular}{@{}l ccc cc ccc @{}}
\toprule
                           & \multicolumn{3}{c}{Sampling}                                   & & \multicolumn{2}{c}{Bias-Tuning (UDDIA \& (e)LIDAO only)}  & & {Mixed Weight (UDDIA \& (e)LIDAO only)}\\ 
\cmidrule{2-4}   \cmidrule{6-7} \cmidrule{9-9}
                           
                           & Probability Coverage   &  Temperature  & Repetition Penalty    & & Learning Rate  & Top Layers to Tune       & & $\tau$\\ 
\midrule
GPT-2                      & 0.9                    & 1.           & 1.                     & & 0.01$^*$       & 18                       & & 0.9\\
OPT                        & 0.9                    & 0.75         & 1.2                    & & 0.006          & 18                       & & 0.8\\
Falcon                     & 0.9                    & 0.75         & 1.2                    & & 0.01           & 18                       & & 0.8\\

\bottomrule

\end{tabular}
}

\end{table*}


\section{More Experimental Results}
\label{app:results}

In this section we present more experimental results.

Table \ref{tab:app:gen} and \ref{tab:app:joint} present full results of using LIDAO for debiasing. 
Here (e)LIDAO-m denotes the min-based (e)LIDAO and (e)LIDAO-p denotes the prod-based (e)LIDAO, respectively. 
We note that the simpler GPT-2 and OPT favored (e)LIDAO-m and Falcon preferred (e)LIDAO-p. 
Nonetheless, both of them achieved comparable debiasing performance compared with baseline approaches. 
Finally, We report results with perplexity evaluated based on larger LMs from the same family in 
Table \ref{tab:app:gen:lpp} and Table \ref{tab:app:joint:lpp}, we note that our conclusions remain valid. 


\begin{table*}[htb!] 
\centering
\caption{
Performance of (e)LIDAO for debiasing the three LMs (shadowed in different colors) with respect to generation $a(\xv)$. 
Bias terms are multiplied by 100 for better comparison. 
Perplexity of all generations are evaluated using GPT-2 XL. 
}
\label{tab:app:gen}
\resizebox{0.75\linewidth}{!}{
\begin{tabular}{@{}l ccc c ccc c ccc@{}}
\toprule
                           & \multicolumn{3}{c}{Sentiment}                                      & & \multicolumn{3}{c}{Regard}                                          & & \multicolumn{3}{c}{Toxicity}\\ 
\cmidrule{2-4}   \cmidrule{6-8} \cmidrule{9-12}
                           
                           & Sent ($\uparrow$)  & Bias ($\downarrow$)   & PPL ($\downarrow$)    & & NNRe ($\uparrow$)   & Bias ($\downarrow$)   & PPL ($\downarrow$)    & & Toxi ($\downarrow$)    & Bias ($\downarrow$)    & PPL ($\downarrow$)\\ 
\midrule
\rowcolor[HTML]{F1F7FF}
GPT-2                      & 0.36$\pm$0.44      & 4.41                  & 22.17                 & & 0.71$\pm$0.55       & 3.89                  & 22.03                 & & 0.16$\pm$0.17          & 9.35                   & 22.17\\ 
\cdashline{1-12}\noalign{\vskip 0.5ex}
\rowcolor[HTML]{F1F7FF} 
+LIDAO-p                   & 0.39$\pm$0.45      & 2.96                  & 31.88                 & & 0.80$\pm$0.50       & 4.35                  & 29.51                 & & 0.13$\pm$0.14          & 13.54                  & 25.68\\
\rowcolor[HTML]{F1F7FF} 
+eLIDAO-p                  & 0.42$\pm$0.46      & 5.12                  & 28.72                 & & 0.77$\pm$0.55       & 2.97                  & 29.20                 & & 0.13$\pm$0.13          & {3.36}                 & {25.44}\\ 
\rowcolor[HTML]{F1F7FF} 
+LIDAO-m                   & 0.40$\pm$0.45      & {0.32}                & {27.49}               & & 0.69$\pm$0.55       & {2.15}                & {28.02}               & & 0.14$\pm$0.15          & 9.88                   & 31.45\\
\rowcolor[HTML]{F1F7FF} 
+eLIDAO-m                  & 0.40$\pm$0.45      & {0.21}                & 28.26                 & & 0.72$\pm$0.55       & {1.53}                & {28.08}               & & 0.14$\pm$0.15          & {5.91}                 & {30.20}\\ 
\midrule

\rowcolor[HTML]{F1FFF1}
OPT                        & 0.33$\pm$0.43      & 0.98                  & 23.62                 & & 0.67$\pm$0.58       & 11.58                 & 23.63                 & & 0.19$\pm$0.20          & 0.78                   & 23.62\\ 
\cdashline{1-12}\noalign{\vskip 0.5ex}
\rowcolor[HTML]{F1FFF1} 
+LIDAO-p                   & 0.41$\pm$0.46      & 0.16                  & 50.50                 & & 0.79$\pm$0.58       & 7.76                  & 49.53                 & & 0.17$\pm$0.17          & 1.29                   & 36.98\\
\rowcolor[HTML]{F1FFF1} 
+eLIDAO-p                  & 0.40$\pm$0.46      & 3.29                  & 47.58                 & & 0.71$\pm$0.63       & 2.85                  & 38.38                 & & 0.17$\pm$0.18          & 7.93                   & 33.10\\ 
\rowcolor[HTML]{F1FFF1} 
+LIDAO-m                   & 0.43$\pm$0.46      & {0.03}                & {40.16}               & & 0.78$\pm$0.62       & {1.09}                & {42.06}               & & 0.16$\pm$0.17          & {0.00}                 & {37.17}\\
\rowcolor[HTML]{F1FFF1}
+eLIDAO-m                  & 0.39$\pm$0.45      & b{2.40}               & {38.69}               & & 0.71$\pm$0.66       & 3.91                  & {34.85}               & & 0.15$\pm$0.17          & 9.10                   & {35.70}\\ 
\midrule

\rowcolor[HTML]{FFF7EE}
Falcon                     & 0.41$\pm$0.46      & 6.39                  & 27.40                 & & 0.65$\pm$0.61       & 2.83                  & 27.00                 & & 0.15$\pm$0.16          & 12.64                  & 27.40\\ 
\cdashline{1-12}\noalign{\vskip 0.5ex}
\rowcolor[HTML]{FFF7EE} 
+LIDAO-p                   & 0.41$\pm$0.46      & {2.96}                & 30.87                 & & 0.70$\pm$0.65       & {2.32}                & 30.07                 & & 0.13$\pm$0.14          & {10.16}                & {28.01}\\
\rowcolor[HTML]{FFF7EE} 
+eLIDAO-p                  & 0.44$\pm$0.46      & {2.74}                & {30.51}               & & 0.75$\pm$0.53       & 5.56                  & 34.16                 & & 0.14$\pm$0.15          & {3.95}                 & {26.92}\\
\rowcolor[HTML]{FFF7EE} 
+LIDAO-m                   & 0.44$\pm$0.46      & 3.95                  & 32.02                 & & 0.68$\pm$0.64       & {2.32}                & {15.18}               & & 0.13$\pm$0.15          & 7.80                   & 30.10\\
\rowcolor[HTML]{FFF7EE} 
+eLIDAO-m                  & 0.45$\pm$0.46      & 6.25                  & 31.19                 & & 0.71$\pm$0.66       & 5.48                  & {15.17}               & & 0.13$\pm$0.15          & 7.80                   & 29.67\\ 

\bottomrule

\end{tabular}
}

\end{table*}

\begin{table*}[htb!] 
\centering
\caption{
Performance of LIDAO for debiasing the three LMs (shadowed in different colors) with respect to joint $a(\cv, \xv)$. 
Bias terms are multiplied by 100 for better comparison. 
Perplexity of all generations are evaluated using GPT-2 XL. 
}
\label{tab:app:joint}
\resizebox{0.75\linewidth}{!}{
\begin{tabular}{@{}l ccc c ccc c ccc@{}}
\toprule
                           & \multicolumn{3}{c}{Sentiment}                                      & & \multicolumn{3}{c}{Regard}                                          & & \multicolumn{3}{c}{Toxicity}\\ 
\cmidrule{2-4}   \cmidrule{6-8} \cmidrule{9-12}
                           
                           & Sent ($\uparrow$)  & Bias ($\downarrow$)   & PPL ($\downarrow$)    & & NNRe ($\uparrow$)   & Bias ($\downarrow$)   & PPL ($\downarrow$)    & & Toxi ($\downarrow$)    & Bias ($\downarrow$)    & PPL ($\downarrow$)\\ 
\midrule
\rowcolor[HTML]{F1F7FF}
GPT-2                      & 0.35$\pm$0.44      & 1.90                  & 23.20                 & & 0.69$\pm$0.56       & 3.66                  & 23.23                 & & 0.14$\pm$0.16           & 16.84                 & 23.20\\ 
\cdashline{1-12}\noalign{\vskip 0.5ex}
\rowcolor[HTML]{F1F7FF} 
+LIDAO-p                   & 0.37$\pm$0.45      & 1.67                  & 29.69                 & & 0.76$\pm$0.53       & 4.93                  & {29.34}               & & 0.11$\pm$0.13           & 2.94                  & 26.85\\
\rowcolor[HTML]{F1F7FF} 
+eLIDAO-p                  & 0.38$\pm$0.45      & 2.12                  & 28.63                 & & 0.75$\pm$0.54       & 4.98                  & {29.24}               & & 0.11$\pm$0.13           & 5.21                  & {26.53}\\ 
\rowcolor[HTML]{F1F7FF} 
+LIDAO-m                   & 0.38$\pm$0.45      & 2.11                  & {28.43}               & & 0.73$\pm$0.54       & {0.29}                & 30.09                 & & 0.12$\pm$0.14           & 8.28                  & 31.27\\
\rowcolor[HTML]{F1F7FF} 
+eLIDAO-m                  & 0.38$\pm$0.45      & {0.46}                & 29.25                 & & 0.73$\pm$0.55       & {0.11}                & 29.79                 & & 0.12$\pm$0.14           & 5.91                  & 30.70\\ 
\midrule

\rowcolor[HTML]{F1FFF1}
OPT                        & 0.34$\pm$0.44      & 1.04                  & 25.52                 & & 0.69$\pm$0.58       & 4.53                  & 25.42                 & & 0.19$\pm$0.20           & 0.78                  & 23.62\\ 
\cdashline{1-12}\noalign{\vskip 0.5ex}
\rowcolor[HTML]{F1FFF1} 
+LIDAO-p                   & 0.41$\pm$0.46      & 2.42                  & 51.01                 & & 0.75$\pm$0.61       & 4.40                  & 51.88                 & & 0.15$\pm$0.17           & 5.77                  & 38.06\\
\rowcolor[HTML]{F1FFF1} 
+eLIDAO-p                  & 0.41$\pm$0.46      & 1.11                  & 50.17                 & & 0.71$\pm$0.63       & 7.53                  & 41.33                 & & 0.15$\pm$0.18           & 2.66                  & 36.06\\ 
\rowcolor[HTML]{F1FFF1} 
+LIDAO-m                   & 0.40$\pm$0.45      & 2.96                  & {43.31}               & & 0.74$\pm$0.62       & 9.09                  & {45.97}               & & 0.14$\pm$0.17           & 6.29                  & 39.72\\
\rowcolor[HTML]{F1FFF1}
+eLIDAO-m                  & 0.39$\pm$0.45      & {0.83}                & {41.38}               & & 0.70$\pm$0.64       & {2.16}                & {36.50}               & & 0.14$\pm$0.17           & 4.80                  & {37.75}\\ 
\midrule

\rowcolor[HTML]{FFF7EE}
Falcon                     & 0.40$\pm$0.45      & 2.28                  & 29.44                 & & 0.70$\pm$0.63       & 1.60                  & 29.24                 & & 0.14$\pm$0.16           & 0.79                  & 29.44\\ 
\cdashline{1-12}\noalign{\vskip 0.5ex}
\rowcolor[HTML]{FFF7EE} 
+LIDAO-p                   & 0.40$\pm$0.45      & 1.61                  & {31.67}               & & 0.73$\pm$0.67       & 3.10                  & {30.64}               & & 0.11$\pm$0.13           & {1.30}                & {30.42}\\
\rowcolor[HTML]{FFF7EE} 
+eLIDAO-p                  & 0.43$\pm$0.46      & 1.57                  & 31.90                 & & 0.76$\pm$0.54       & {2.43}                & 39.47                 & & 0.12$\pm$0.15           & {1.03}                & {30.99}\\
\rowcolor[HTML]{FFF7EE} 
+LIDAO-m                   & 0.41$\pm$0.45      & 2.75                  & {31.04}               & & 0.68$\pm$0.64       & 3.38                  & 31.79                 & & 0.11$\pm$0.14           & 5.28                  & 32.79\\
\rowcolor[HTML]{FFF7EE} 
+eLIDAO-m                  & 0.41$\pm$0.45      & 2.71                  & 32.29                 & & 0.72$\pm$0.55       & 3.51                  & 35.51                 & & 0.11$\pm$0.15           & 8.05                  & 32.43\\ 

\bottomrule

\end{tabular}
}

\end{table*}

\begin{table*}[htb!] 
\centering
\caption{
Performance of debiasing the three LMs (shadowed in different colors) with respect to generation $a(\xv)$. 
Bias terms are multiplied by 100 for better comparison. 
Perplexity of OPT and Falcon generations are evaluated using OPT-13B and Falcon-40B respectively. 
}
\label{tab:app:gen:lpp}
\resizebox{0.75\linewidth}{!}{
\begin{tabular}{@{}l ccc c ccc c ccc@{}}
\toprule
                           & \multicolumn{3}{c}{Sentiment}                                      & & \multicolumn{3}{c}{Regard}                                          & & \multicolumn{3}{c}{Toxicity}\\ 
\cmidrule{2-4}   \cmidrule{6-8} \cmidrule{9-12}
                           
                           & Sent ($\uparrow$)  & Bias ($\downarrow$)   & PPL ($\downarrow$)    & & NNRe ($\uparrow$)   & Bias ($\downarrow$)   & PPL ($\downarrow$)    & & Toxi ($\downarrow$)    & Bias ($\downarrow$)   & PPL ($\downarrow$)\\ 
\midrule
\rowcolor[HTML]{F1FFF1}
OPT                        & 0.33$\pm$0.43      & 0.98                  & 10.93                 & & 0.67$\pm$0.58       & 11.58                 & 10.91                 & & 0.19$\pm$0.20          & 0.78                  & 10.93\\ 
\cdashline{1-12}\noalign{\vskip 0.5ex}
\rowcolor[HTML]{F1FFF1} 
+PPLM                      & 0.40$\pm$0.46      & 4.77                  & 27.33                 & & 0.72$\pm$0.63       & 5.10                  & 27.57                 & & 0.17$\pm$0.18          & 2.27                  & {18.92}\\
\rowcolor[HTML]{F1FFF1} 
+UDDIA                     & 0.39$\pm$0.45      & 4.88                  & 38.69                 & & 0.74$\pm$0.63       & {1.98}                & 39.90                 & & 0.16$\pm$0.17          & {1.69}                & 20.67\\ 
\cdashline{1-12}\noalign{\vskip 0.5ex}
\rowcolor[HTML]{F1FFF1} 
+LIDAO-p                   & 0.41$\pm$0.46      & 0.56                  & 28.56                 & & 0.79$\pm$0.58       & 7.76                  & 26.94                 & & 0.17$\pm$0.17          & 1.29                  & 17.56\\
\rowcolor[HTML]{F1FFF1} 
+eLIDAO-p                  & 0.40$\pm$0.46      & 3.29                  & 25.21                 & & 0.71$\pm$0.63       & 2.92                  & 28.91                 & & 0.17$\pm$0.18          & 7.93                  & 16.07\\ 
\rowcolor[HTML]{F1FFF1} 
+LIDAO-m                   & 0.43$\pm$0.46      & {0.06}                & {20.21}               & & 0.78$\pm$0.62       & {1.16}                & {22.25}               & & 0.16$\pm$0.17          & {0.00}                & {18.69}\\
\rowcolor[HTML]{F1FFF1}
+eLIDAO-m                  & 0.39$\pm$0.45      & {2.61}                & {20.39}               & & 0.71$\pm$0.66       & 3.91                  & {25.19}               & & 0.15$\pm$0.17          & 9.10                  & {18.38}\\ 
\midrule

\rowcolor[HTML]{FFF7EE}
Falcon                     & 0.41$\pm$0.46      & 6.39                  & 11.94                 & & 0.65$\pm$0.61       & 2.83                  & 11.95                 & & 0.15$\pm$0.16          & 12.64                 & 11.94\\ 
\cdashline{1-12}\noalign{\vskip 0.5ex}
\rowcolor[HTML]{FFF7EE} 
+PPLM                      & 0.46$\pm$0.46      & {1.42}                & 45.99                 & & 0.72$\pm$0.68       & 8.03                  & 48.35                 & & 0.13$\pm$0.14          & 21.79                 & 35.19\\
\rowcolor[HTML]{FFF7EE} 
+UDDIA                     & 0.47$\pm$0.46      & 4.54                  & {14.92}               & & 0.77$\pm$0.66       & {2.56}                & {14.56}               & & 0.12$\pm$0.14          & 14.94                 & 13.57\\ 
\cdashline{1-12}\noalign{\vskip 0.5ex}
\rowcolor[HTML]{FFF7EE} 
+LIDAO-p                   & 0.41$\pm$0.46      & {2.96}                & 15.53                 & & 0.70$\pm$0.65       & {2.59}                & {15.84}               & & 0.13$\pm$0.14          & {10.16}                & {12.77}\\
\rowcolor[HTML]{FFF7EE} 
+eLIDAO-p                  & 0.44$\pm$0.46      & {2.74}                & {15.42}               & & 0.75$\pm$0.53       & 7.61                  & 27.76                 & & 0.14$\pm$0.15          & {3.95}                 & {12.97}\\
\rowcolor[HTML]{FFF7EE} 
+LIDAO-m                   & 0.44$\pm$0.46      & 3.95                  & 15.98                 & & 0.68$\pm$0.64       & {2.32}                & {15.18}               & & 0.13$\pm$0.15          & 7.80                   & 14.00\\
\rowcolor[HTML]{FFF7EE} 
+eLIDAO-m                  & 0.45$\pm$0.46      & 6.25                  & 16.50                 & & 0.71$\pm$0.66       & 5.48                  & {15.17}               & & 0.13$\pm$0.15          & 7.80                   & 14.33\\ 
\bottomrule

\end{tabular}
}
\end{table*}

\begin{table*}[htb!] 
\centering
\caption{
Performance of debiasing the three LMs (shadowed in different colors) with respect to joint $a(\cv, \xv)$. 
Bias terms are multiplied by 100 for better comparison. 
Perplexity of OPT and Falcon generations are evaluated using OPT-13B and Falcon-40B respectively. 
}
\label{tab:app:joint:lpp}
\resizebox{0.75\linewidth}{!}{
\begin{tabular}{@{}l ccc c ccc c ccc@{}}
\toprule
                           & \multicolumn{3}{c}{Sentiment}                                      & & \multicolumn{3}{c}{Regard}                                          & & \multicolumn{3}{c}{Toxicity}\\ 
\cmidrule{2-4}   \cmidrule{6-8} \cmidrule{9-12}
                           
                           & Sent ($\uparrow$)  & Bias ($\downarrow$)   & PPL ($\downarrow$)    & & NNRe ($\uparrow$)   & Bias ($\downarrow$)   & PPL ($\downarrow$)    & & Toxi ($\downarrow$)     & Bias ($\downarrow$)   & PPL ($\downarrow$)\\ 
\midrule

\rowcolor[HTML]{F1FFF1}
OPT                        & 0.34$\pm$0.44      & 1.04                  & 10.92                 & & 0.69$\pm$0.58       & 4.53                  & 10.94                 & & 0.18$\pm$0.20           & 0.78                  & 10.92\\ 
\cdashline{1-12}\noalign{\vskip 0.5ex}
\rowcolor[HTML]{F1FFF1} 
+PPLM                      & 0.40$\pm$0.46      & 4.38                  & 28.89                 & & 0.70$\pm$0.65       & 4.12                  & 27.12                 & & 0.16$\pm$0.18           & {1.78}                & {18.96}\\
\rowcolor[HTML]{F1FFF1} 
+UDDIA                     & 0.40$\pm$0.45      & {1.03}                & 38.11                 & & 0.70$\pm$0.61       & {2.17}                & 41.48                 & & 0.15$\pm$0.17           & {2.94}                & 21.54\\ 
\cdashline{1-12}\noalign{\vskip 0.5ex}
\rowcolor[HTML]{F1FFF1} 
+LIDAO-p                   & 0.41$\pm$0.46      & 2.49                  & 27.63                 & & 0.75$\pm$0.61       & 4.32                  & 27.06                 & & 0.15$\pm$0.17           & 5.77                  & 17.35\\
\rowcolor[HTML]{F1FFF1} 
+eLIDAO-p                  & 0.41$\pm$0.46      & 1.11                  & 26.17                 & & 0.71$\pm$0.63       & 7.83                  & 28.76                 & & 0.15$\pm$0.18           & 2.66                  & 16.86\\ 
\rowcolor[HTML]{F1FFF1} 
+LIDAO-m                   & 0.40$\pm$0.45      & 2.93                  & {20.56}               & & 0.74$\pm$0.62       & 9.19                  & {22.62}               & & 0.14$\pm$0.17           & 6.29                  & 19.49\\
\rowcolor[HTML]{F1FFF1}
+eLIDAO                    & 0.39$\pm$0.45      & {0.56}                & {20.68}               & & 0.70$\pm$0.64       & {2.16}                & {25.32}               & & 0.14$\pm$0.17           & 4.80                  & {18.45}\\ 
\midrule

\rowcolor[HTML]{FFF7EE}
Falcon                     & 0.40$\pm$0.45      & 2.28                  & 12.45                 & & 0.70$\pm$0.63       & 1.60                  & 12.43                 & & 0.14$\pm$0.16           & 0.79                  & 12.45\\ 
\cdashline{1-12}\noalign{\vskip 0.5ex}
\rowcolor[HTML]{FFF7EE} 
+PPLM                      & 0.46$\pm$0.46      & {1.04}                & 47.93                 & & 0.72$\pm$0.68       & 6.59                  & 49.08                 & & 0.11$\pm$0.13           & 19.38                 & 36.64\\
\rowcolor[HTML]{FFF7EE} 
+UDDIA                     & 0.45$\pm$0.46      & {0.81}                & {14.63}               & & 0.75$\pm$0.66       & {1.90}                & {14.89}               & & 0.11$\pm$0.13           & 3.97                  & 14.29\\ 
\cdashline{1-12}\noalign{\vskip 0.5ex}
\rowcolor[HTML]{FFF7EE} 
+LIDAO-p                   & 0.40$\pm$0.45      & 1.61                  & {15.73}               & & 0.73$\pm$0.67       & 3.10                  & {15.73}               & & 0.11$\pm$0.13           & {1.30}                & {13.41}\\
\rowcolor[HTML]{FFF7EE} 
+eLIDAO-p                  & 0.43$\pm$0.46      & {1.57}                & 16.10                 & & 0.76$\pm$0.54       & {2.43}                & 28.31                 & & 0.12$\pm$0.15           & {1.03}                & {13.94}\\
\rowcolor[HTML]{FFF7EE} 
+LIDAO-m                   & 0.41$\pm$0.45      & 2.75                  & {15.12}               & & 0.68$\pm$0.64       & 3.38                  & 14.89                 & & 0.11$\pm$0.14           & 5.28                  & 14.75\\
\rowcolor[HTML]{FFF7EE} 
+eLIDAO-m                  & 0.41$\pm$0.45      & 2.71                  & 16.03                 & & 0.72$\pm$0.55       & 3.51                  & 15.20                 & & 0.11$\pm$0.15           & 8.05                  & 14.96\\ 
\bottomrule

\end{tabular}
}

\end{table*}